\documentclass[sn-mathphys-num]{sn-jnl}

\usepackage{graphicx}\usepackage{multirow}\usepackage{amsmath,amssymb,amsfonts}\usepackage{amsthm}\usepackage{mathrsfs}\usepackage[title]{appendix}\usepackage{xcolor}\usepackage{textcomp}\usepackage{manyfoot}\usepackage{booktabs}\usepackage{listings}
\usepackage{tikz}
\usetikzlibrary{calc}
\usetikzlibrary{backgrounds,positioning}
\usetikzlibrary{decorations.pathreplacing}

\usepackage{diagbox}
\usepackage{comment}

\usepackage{svg}

\usepackage{xcolor}
\usepackage[linesnumbered,ruled,vlined,procnumbered,lined,boxed]{algorithm2e}
\usepackage{etoolbox}

\makeatletter
\AtBeginEnvironment{procedure}{\let\c@algocf\c@procedure}
\usepackage{amsmath}

\SetCommentSty{mycommfont}

\theoremstyle{thmstyleone}
\theoremstyle{thmstyletwo}
\theoremstyle{thmstylethree}
\newtheorem{lemma}{Lemma}

\begin{document}

\title[Adaptive Opponent Policy Detection in Multi-Agent MDPs: Real-Time Strategy Switch Identification Using Running Error Estimation]{Adaptive Opponent Policy Detection in Multi-Agent MDPs: Real-Time Strategy Switch Identification Using Running Error Estimation}

\author[1]{\fnm{Mohidul Haque Mridul}}\email{mohidulhaque-2018125308@cs.du.ac.bd}

\author[1]{\fnm{Mohammad Foysal} \sur{Khan}}\email{mohammadfoysal-2018225299@cs.du.ac.bd}

\author*[1]{\fnm{Redwan Ahmed} \sur{Rizvee}}\email{rizvee@cse.du.ac.bd}

\author*[1]{\fnm{Md Mosaddek} \sur{Khan}}\email{mosaddek@du.ac.bd}

\affil[1]{\orgdiv{Department of Computer Science and Engineering}, \orgname{University of Dhaka}, \orgaddress{\state{Dhaka}, \country{Bangladesh}}}

\abstract{In Multi-agent Reinforcement Learning (MARL), accurately perceiving opponents' strategies is essential for both cooperative and adversarial contexts, particularly within dynamic environments. While Proximal Policy Optimization (PPO) and related algorithms such as Actor-Critic with Experience Replay (ACER), Trust Region Policy Optimization (TRPO), and Deep Deterministic Policy Gradient (DDPG) perform well in single-agent, stationary environments, they suffer from high variance in MARL due to non-stationary and hidden policies of opponents, leading to diminished reward performance. Additionally, existing methods in MARL face significant challenges, including the need for inter-agent communication, reliance on explicit reward information, high computational demands, and sampling inefficiencies. These issues render them less effective in continuous environments where opponents may abruptly change their policies without prior notice. Against this background, we present OPS-DeMo (Online Policy Switch-Detection Model), an online algorithm that employs dynamic error decay to detect changes in opponents' policies. OPS-DeMo continuously updates its beliefs using an Assumed Opponent Policy (AOP) Bank and selects corresponding responses from a pre-trained Response Policy Bank. Each response policy is trained against consistently strategizing opponents, reducing training uncertainty and enabling the effective use of algorithms like PPO in multi-agent environments. Comparative assessments show that our approach outperforms PPO-trained models in dynamic scenarios like the Predator-Prey setting, providing greater robustness to sudden policy shifts and enabling more informed decision-making through precise opponent policy insights.}


\maketitle

\section{Introduction}
In real-world scenarios, artificial agents face significant challenges in navigating interactions among multiple entities, a task that humans manage with apparent ease \cite{humandigitaltwin}. Examples include soccer players who must predict the movements of teammates, each with varying roles and skill sets, as well as opponents, and autonomous cars that need to anticipate the diverse behaviors of other vehicles on the road \cite{carautonomy}. These encounters involve distinct behaviors, necessitating different policies for optimal outcomes. Termed opponents, these entities create a non-stationary environment from a decentralized agent's perspective.

In normal situations, it is common for the strategies of opponents or interacting entities to be hidden. Despite the potential benefits of exchanging policies, frequent communication is not always feasible, and opponents may be reluctant to share their strategies, hindering direct learning opportunities \cite{policyconcealment}. Consequently, there is a need to update our beliefs about opponents' policies by observing their actions in real-time. However, relying solely on past observations can be problematic, as these may not accurately reflect the current policy owing to random fluctuations or adaptations in the opponent's learning process. By placing greater emphasis on recent actions, we can gradually converge toward understanding their true current policy. However, abrupt shifts in opponent policies can occur, leading to challenges in accurately tracking their strategies. Failing to detect these sudden changes could slow down learning progress and hinder the agent's ability to adapt effectively. An example can be observed in analyzing the financial market, where, the policies of the players are often not expressed explicitly, and based on various factors from small to large, the environment of the market can abruptly shift.

Interacting effectively with other agents requires understanding their actions and decisions, a process aided by opponent modeling - constructing and using models of opponents' behavior. Ideally, opponent models help extract opponent policies, aiding in devising strategies exploiting opponents' weaknesses. However, due to the dynamic nature of agents' behaviors, such modeling is challenging. Non-stationary behaviors require continuous updating of learned opponent models. For example, in soccer, a defender can become an attacker during a counter-attack, altering their strategy significantly. Similarly, autonomous agents may adjust their policies based on changing beliefs about the environment.

To address the aforementioned concerns, Reinforcement learning, a paradigm for policy learning has been used that focuses on agents' maximizing long-term cumulative reward through trial-and-error with the environment \cite{Andrew_1999}. While effective in single-agent environments, it struggles in multi-agent settings due to the non-stationarity of agents. Common reinforcement learning algorithms such as DQN, DDPG, AAC, and PPO are optimized to excel in achieving high rewards within static environments known as Markov Decision Processes (MDP). However, when applied to Multi-Agent MDPs, where multiple agents interact, these algorithms encounter challenges due to increased variance caused by gradual or sudden changes in other agents' policies.

Several algorithms tailored for multi-agent scenarios, such as BPR+\cite{bprplus}, DPN-BPR+\cite{dpnbprplus}, LOLA\cite{LOLA}, and meta-MAPG\cite{metamapg}, address the challenges of non-stationary environments through various strategies. LOLA aims to influence opponents' behavior, while Meta-MAPG focuses on meta-learning environmental dynamics. BPR+ and DPN-BPR+ leverage previously learned response policies, assuming infrequent non-stationarity in opponents, akin to occasional shifts between stationary opponents. However, LOLA and Meta-MAPG are inadequate for sudden shifts in opponent policy. In contrast, BPR+ and DPN-BPR+ rely on reward signals to detect opponent policy shifts, limiting their effectiveness to episodic environments with consistent reward data, posing challenges in continuous environments for accurately detecting sudden policy switches.

The need for an effective method to detect policy switches based solely on observed actions is critical, particularly when traditional statistical measures fall short in scenarios with brief opponent trajectories.  SAM (Switching Agent Model) addresses this by estimating the running error of the assumed policy, preventing an indefinite error increase when the agent follows the assumed policy. However, SAM is tailored for deterministic actions, often paired with DDPG \cite{ddpg}, and lacks detailed formulations for running error decay.

Building upon the SAM framework, our approach extends its applicability to Proximal Policy Optimization (PPO) and similar reinforcement-learning algorithms. We rectify the gaps in SAM by detailing the decay calculations and enhancing the method for selecting response policies. Specifically, we propose a method to detect policy shifts in real-time by reusing response policies on a fixed set of opponent policies, assuming occasional switches between them. This method relies solely on observed actions, employing running error estimation and dynamic error decay adaptable to stochastic policies. Upon detecting a switch, our algorithm swiftly adjusts the response policy to align with the most probable opponent policy. To accommodate the resource constraints typical in edge devices, we optimize our policy switch detection mechanism for efficiency, allowing it to operate within strict resource limitations and process observations on the fly, without the need for storing them.

To summarize the main contribution of this work, we highlight the following points,

\begin{enumerate}
    \item
We introduce a \textit{running error estimation metric} to assess an agent's compliance with a stochastic policy, utilizing \textit{only observed state-action pairs} from the trajectory. This metric operates \textit{online} and can be continually updated as new observations are processed, \textit{without the need for storage}.
\item Additionally, we propose an online algorithm that utilizes the running error estimation metric to detect the policy switch of an opponent and adapt the response policy accordingly.

\item Through conducting a rigorous comparative analysis between our proposed algorithm and the current state-of-the-art algorithms, we evaluate the merit of the proposals and present them.
\end{enumerate}

In Section \ref{sec:background}, we present a summary of our related works. In Section \ref{sec:proposals}, we present our proposals and an experimental evaluation regarding the merit of the work is presented in Section \ref{sec:evaluation}. The article ends with some concluding marks in Section \ref{sec:conclusion}.
 \section{Related Works} \label{sec:background}

In this section, we explore topics closely related to our work, beginning with a fundamental concept: the Markov Decision Process (MDP), which models decision-making under uncertainty by defining states, actions, transition probabilities, and rewards. It assumes the Markov property, where future states depend only on the current state and action. The goal is to find a policy that maximizes expected cumulative rewards. In the reinforcement learning paradigm, MDPs are often modeled under single agent-based framework. Markov Games are an extension of MDP which also share a similar set of terminologies and objectives, such as states, actions, transition probabilities, and rewards in conjunction with maximizing the cumulative reward \cite{markovgames}.

In Markov Games, the concept of multiple agents operating within the same environment is introduced, each agent focused on maximizing its own reward objective. These kind of games are effectively dealt with Multi-Agent Reinforcement Learning (MARL) approaches \cite{marl}. MARL extends single-agent reinforcement learning techniques to handle interactions among multiple agents. Agents learn individual policies based on observations and joint rewards. Common MARL algorithms include Independent Q-Learning \cite{qlearning}, Q-Learning with Experience Replay \cite{qler}, and Multi-Agent Deep Deterministic Policy Gradient (MADDPG) \cite{modelingwithmarl}. MARL enables agents to learn effective strategies in dynamic environments where their actions affect the rewards and states of other agents. It facilitates the exploration of cooperative, competitive behaviors, and emergent phenomena in multi-agent systems.

In this context, Multi-Agent Markov Decision Processes (MMDPs) serve as a natural extension of both Markov Games and Multi-Agent Reinforcement Learning (MARL). By incorporating the concepts of states, actions, transition probabilities, and rewards, MMDPs provide a comprehensive framework for modeling decision-making in multi-agent environments. Q-Learning \cite{qlearning, qler}, (MADDPG) \cite{modelingwithmarl}, etc. falls under the wide umbrella of solutions for MMDPs. Modeling through MMDPs provide invaluable insights into the dynamics of interactions among autonomous agents, enabling the development of intelligent and adaptive systems capable of functioning effectively in real-world scenarios.

In reinforcement learning, algorithms like REINFORCE \cite{reinforce}, DQN \cite{dqn}, AAC \cite{aac}, DDPG \cite{ddpg}, SAC \cite{sac}, TRPO \cite{trpo}, PPO \cite{ppo}, etc., are widely studied for tasks utilizing techniques such as value iteration, policy iteration, Q-learning, and policy gradient methods. However, these algorithms are mainly tailored for stationary environments with single agents, often represented by MDPs. Real-world scenarios (MMDPs) often involve multiple interacting agents, leading to non-stationarity and increased variance, posing challenges for model convergence and achieving optimal results.

PPO \cite{ppo}, a policy gradient method, uses stochastic gradient ascent to optimize a surrogate objective function while iteratively collecting sample data from the environment. Empirically, PPO performs comparably to or better than other policy gradient algorithms such as AAC and TRPO in continuous control environments. However, PPO does not address variance introduced by actions of other agents within the same environment, as it is designed solely for single-agent scenarios. Numerous studies have focused on explicitly modeling the behaviors, goals, and beliefs of interacting agents (referred to as opponents) to accommodate their dynamics effectively \cite{yu2022surprising}. This approach deviates from treating agents solely as elements of the environment, aiming to achieve a more adaptable and robust policy by incorporating explicit modeling of agents.

It is crucial to recognize that opponent behavior may be non-stationary due to factors such as opponents actively learning or modeling our agents' behavior. Many studies overlook this concept, failing to adequately address the potential non-stationarity of opponent behavior. Neglecting this aspect can limit the effectiveness of approaches in capturing the true complexity of interactive dynamics. Several algorithms, such as RL-CD \cite{rlcd} and QCD \cite{qcd}, are designed to handle non-stationary opponents by positing that they periodically switch between multiple stationary strategies. However, RL-CD and QCD primarily focus on accurately detecting switch points rather than explicitly identifying opponent strategies. Nonetheless, identifying the opponent's strategy is crucial in addition to detecting switch points.

The algorithms mentioned treat the non-stationary opponent as part of the environment and do not explicitly learn an opponent model. Although this simplifies the algorithm, it can lead to reduced performance. To address this limitation, other algorithms such as MDP-CL \cite{mdpcl}, DriftER \cite{drifter}, and BPR+ \cite{bprplus} incorporate explicit opponent modeling. MDP-CL dynamically explores opponents during runtime without prior models, switching strategies by comparing modeled policies with observed behavior. However, it does not store the modeled policy after the switch. DriftER incorporates drift for switch detection and uses the R-max algorithm instead of random exploration.
BPR+ avoids the need for an MDP and employs cumulative rewards and Bayes' rule to detect switch points and opponent strategies. It stores learned policies in memory. These methods assume opponents will maintain a stationary policy for several episodes, as frequent switching would hinder opponent modeling and make the opponent's strategy appear random.

DPN-BPR+ \cite{dpnbprplus} introduces a methodology for detecting policy shifts using both reward signals and opponent behaviors, utilizing a policy bank to generate appropriate responses. Its applicability is limited to episodic environments, lacks guidance on constructing the policy bank using unsupervised methods and does not address real-time improvement of responses against known policies. LOLA \cite{LOLA} assumes opponents to be naively learning and aims to shape their learning to the agent's advantage, although opponents may resist such shaping. Meta-PG \cite{metapg} and meta-MAPG \cite{metamapg} approach continuous adapt through meta-learning, with meta-MAPG combining LOLA with meta-PG for multiple agents. M-FOS \cite{mfos} learns meta-policies for long-term opponent shaping and exploitation of LOLA, while MBOM \cite{mbom} employs recursive imagination and Bayesian mixing to generate responses. However, these approaches do not account for abrupt policy shifts by opponents, which can be challenging to predict.

Recent advancements in Deep Reinforcement Learning (DRL) have addressed non-stationarity in multi-agent systems. Algorithms like DRON \cite{dron}, MADDPG \cite{maddpg}, and DPIQN \cite{dpiqn} handle this issue by learning generalized policies from predefined features or observed opponent behaviors. DPIQN learns distinct features for opponent policies but trains a generalized Q-network for execution instead of reusing advantageous response policies. SAM \cite{EverettR18} models opponent policies to adapt response policies effectively, integrating DDPG with opponent modeling. It assesses opponent compliance using a running error metric and infers policy switches when the error surpasses a threshold. However, SAM lacks explicit definitions for determining the next response strategy and the decay mechanism for the running error.

In summary, to address the aforementioned limitations observed in different literature, in this work, we propose an online continuous running error estimation metric utilizing only observed state-action pairs and an online algorithm to detect the policy switch of the opponents to adapt responses.
\section{The Online Policy Switch Detection Model (OPS-DeMo)} \label{sec:proposals}
In this section, we present our proposals to address the issues discussed in Section \ref{sec:background}. First, we introduce a new metric to measure how well an agent complies with policies, based on its recent actions. This metric can be used on the fly and is detailed in Section \ref{sec:metric_policy}. Next, we describe the architecture of our proposed model in Section \ref{sec:architecture}. Finally, we introduce an algorithm in Section \ref{sec:algorithm_discuss} that is specifically tailored to adapt to changes in the opponent’s behavior.

\subsection{Metric to Measure Policy Compliance} \label{sec:metric_policy}

Detecting compliance in policies with nearly uniform action probability distributions across Markov states poses significant challenges, particularly in environments characterized by short trajectories. Traditional methods, which rely on frequency distribution, often fall short as they require frequent revisits to states—a condition that is impractical with limited data availability. A more viable approach entails comparing observed actions against their expected probabilities and calculating an error metric in real-time. Should this error exceed a predefined threshold, it suggests potential deviation from the policy. Nonetheless, to mitigate error accumulation stemming from inherent randomness, implementing a decay mechanism is essential.

This decay mechanism should consider both the expected error when the agent adheres to the policy and when it deviates from it. By incorporating this decay, the method aims to prevent error escalation indefinitely, especially in scenarios where the agent genuinely follows a stochastic policy with inherent sampling errors. Let the assumed policy for an MDP with a discrete action space be $\pi$. In a given Markov state $s$, the policy $\pi$ can be written as per the Equation \ref{eqn:prob_action}. Here $p_{a_i}$ denotes the probability of choosing action $a_i$ from state $s$.

\begin{equation}
\begin{aligned}
\pi(s) = \begin{bmatrix}
    p_{a_1}, p_{a_2}, p_{a_3}, \ldots p_{a_i}, \ldots ,p_{a_n}
\end{bmatrix}
\end{aligned}
\label{eqn:prob_action}
\end{equation}

Similarly, the observed frequency of actions in a given Markov state $s$ can be written as per Equation \ref{eqn:freq_dist}. Here, in Equation \ref{eqn:freq_dist}, $f_{a_i}$ is set to $1$ when the action $a_i$ is chosen. Otherwise, it will be set to $0$.

\begin{equation}
\begin{aligned}
f_o(s) = \begin{bmatrix}
    f_{a_1}, f_{a_2}, f_{a_3} \ldots f_{a_i}, \ldots ,f_{a_n}
\end{bmatrix}
\end{aligned}
\label{eqn:freq_dist}
\end{equation}

Now, based on Equations \ref{eqn:prob_action} and \ref{eqn:freq_dist}, the observed error in state $s$ from assuming the agent follows policy $\pi$ can be written as,

\begin{equation}
\begin{aligned}
e_o(\pi, s) = \frac{1}{2}\sum_{k=1}^{n}|\pi(s) - f_o(s)|_{a_k}
\end{aligned}
\label{eqn:running_error}
\end{equation}

Now, we state some lemmas to discuss some features associated with Equation \ref{eqn:running_error}.

\begin{lemma} \label{lemma:running_error_simp}
Consider a timestep $t$ within the context of a MDP with a discrete action space of $n$ actions, wherein an agent follows a policy $\pi$ and selects an action $a_i$ from a Markovian state $s$. Within this framework, the observed error at $t$ can be formulated as $(1 - p_{a_i})$, where $p_{a_i}$ signifies the probability of opting for action $a_i$ by the stochastic policy $\pi$.
\end{lemma}

\begin{proof}
Let us examine the observed frequency of each action, noting that all actions except $a_i$ have a frequency of 0. Consequently, the observed error can be expressed as follows:

\begin{equation}
\begin{aligned}
e_o(\pi, s) &= \frac{1}{2}[ |0-p_{a_1}| + |0-p_{a_2}| + |0-p_{a_3}| + \ldots  |1 - p_{a_i}| + |0- p_{a_{i+1}}| + \ldots +|0-p_{a_n}| ]\\
&\text{(From Equation \ref{eqn:running_error})}\\
&= \frac{1}{2} [p_{a_1} + p_{a_2} + p_{a_3} + \ldots + (1-p_{a_i}) + p_{a_{i+1}} + \ldots p_{a_n}]\\
&\text{(Because, any $0 \leq p_{a_i} \leq 1$)} \\
&= \frac{1}{2}[(1-p_{a_i}) + (1-p_{a_i})] \\
&\text{(Because, any $\sum_{j=1}^n p_{a_j} = 1 \Rightarrow p_{a_i} = 1-\sum_{k=1, k\neq i}^{n} p_{a_k}$)} \\
&= (1-p_{a_i})
\end{aligned}
\end{equation}
\end{proof}

Since observed error $e_o(\pi, s)$ has occurred due to the selected action $a$, in this discussion, we  also use $e_o(\pi, s, a)$ to denote the similar meaning of observed error from state $s$ due to an action $a$ under policy $\pi$.

\begin{lemma}
In a MDP, with a discrete action space of $n$ actions, consider a time step $t$ where an agent follows a policy $\pi$, and the system is in a Markovian state $s$. Within this framework, the expected naturally occurring error from policy $\pi$ at $t$ step can be expressed as $\sum_{j=1}^n p_{a_j}(1 - p_{a_j})$. Here, $p_{a_j}$ represents the probability of selecting action $a_j$ following the stochastic policy $\pi$.
\end{lemma}

\begin{proof}
The probability of selecting action $a_i$ is $p_{a_i}$ when the agent is following policy $\pi$, which induces observed error $(1 - p_{a_i})$ (Lemma \ref{lemma:running_error_simp}). Therefore the expected error when following policy $\pi$ is:
\begin{equation}
\label{eq:following_policy}
\begin{aligned}
E\Bigg(\frac{e_o(\pi, s) }{ \pi}\Bigg) = \sum_{j=1}^n p_{a_j}(1 - p_{a_j})
\end{aligned}
\end{equation}
\end{proof}

\begin{lemma}
In a MDP context, with a discrete action space of $n$ actions, consider a time step $t$ where an agent follows any policy $\phi$ except a certain policy $\pi$, and the system is in a Markovian state $s$. Within this framework, the expected naturally occurring error from policy $\pi$ at $t$ can be expressed as $\frac{n-1}{n}$.
\end{lemma}

\begin{proof}
Consider a scenario where the agent deviates from policy $\pi$ and instead complies with an alternative but unknown policy $\phi$. In this context, the probability of the agent selecting any action while not following policy $\pi$ is uniformly distributed across all actions, although the specific distribution under policy $\phi$ remains unknown. Consequently, the expected observed error from policy $\pi$, when not following policy $\pi$, can be articulated as Equation \ref{eq:not_following_policy}. Here $\pi^c$ denotes the set of all possible policies for the particular problem.

\begin{equation}
\label{eq:not_following_policy}
\begin{aligned}
E\left(\frac{e_o(\pi,s)}{\phi \in \pi^{c}}\right) = \sum_{j=1}^n \frac{1}{n}(1 - p_{a_j}) = \frac{n-1}{n}
\end{aligned}
\end{equation}

\end{proof}

\subsection{Architecture of the Model} \label{sec:architecture}

\begin{figure}[ht]
\centering
\includegraphics[width=0.6\textwidth]{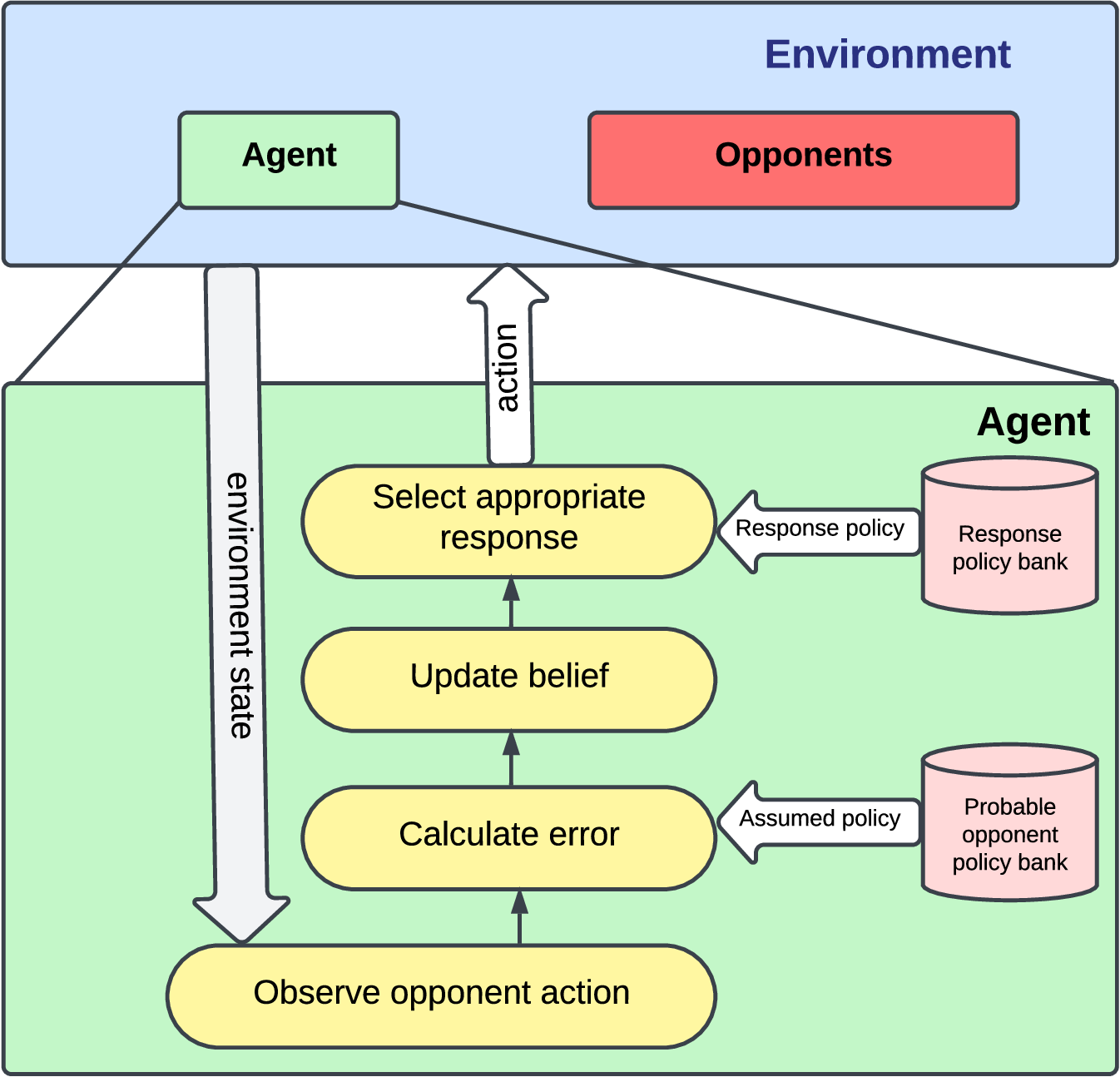}
\caption{Architecture of OPS-DeMo}
\label{fig:architecture}
\end{figure}

In the given environment modeled by an MDP with a discrete action space, we adopt a strategy of training response policies $\pi_i$ against each of the potential opponent policies $\Phi_i \in \Pi_o$, where $\Pi_o$ denotes the policy bank containing the various probable opponent policies. The training of response policies involves employing state-of-the-art learning algorithms such as PPO. Following the training process, our agent is prepared for deployment within the environment. Incorporating all the ideas, we have given a high-level overview of our proposed OPS-DeMo architecture in Fig. \ref{fig:architecture}.

At any given instance when the opponent is presumed to follow a specific policy $\phi_i$ and is observed selecting action $a_j$, we leverage the policy bank $\Pi_o$ to ascertain the probability of each probable opponent policy choosing the observed action. Subsequently, we calculate the corresponding observed error and update our belief regarding the current opponent policy utilizing a designated algorithm. After this belief update, we select an appropriate response policy and determine the agent's action based on the chosen response policy. This iterative process allows our agent to adapt dynamically to varying opponent strategies encountered during its deployment.

\subsection{Algorithm Description} \label{sec:algorithm_discuss}
We propose an algorithm (Algorithm~\ref{alg:opsdemo}) for running error estimation (lines 8-17), policy switch detection (lines 18-19), and adapting the response policy (lines 20-21). This algorithm utilizes the already trained opponent's policy bank $\Pi_{o}$ and the PPO-trained response policy bank $\rho$ according to the updated beliefs about the opponent's current policy to maximize accumulated rewards. The algorithm uses the provided running error estimation method to see which policy from $\Pi_o$ more aligns with the opponent's recent actions and then picks the policy with the lowest running error, and the appropriate response policy is taken to maximize rewards in the current scenario.

\begin{algorithm}[!ht]
\caption{OPS-DeMo}
\label{alg:opsdemo}
\SetKwInOut{Input}{Input}
\Input{Observed state-action pairs, Opponent's Probable Policy Bank $\Pi_o$, PPO-Trained Response Policy Bank $\rho$}
\textbf{Initialization:}\\
Initialize Opponent's Probable Policy Bank $\Pi_o$\;
Initialize PPO-Trained Response Policy Bank $\rho$\;
Randomly choose an assumed opponent policy $\phi_{\text{assumed}}$ from $\Pi_o$\;
Initialize running errors for all opponent policies: $running\_error(\phi) \gets 0$ for all $\phi \in \Pi_o$\;
Initialize threshold value for policy switch detection\;
\SetAlgoLined
\ForEach{observed state-action pair $(s,a)$ of the opponent}{
    \ForEach{opponent policy $\phi \in \Pi_o$}{
        Calculate Observed Error $e_{o}(\phi, s, a)$ (Lemma~\ref{lemma:running_error_simp}), expected error $e_f$ when following the assumed policy $\phi$ (Equation~\ref{eq:following_policy}), and $e_{nf}$ while not following the assumed policy (Equation~\ref{eq:not_following_policy}) using the provided equations\;
        $decay(s, \phi) \gets \alpha \times e_f + (1 - \alpha) \times e_{nf}$\;
        $running\_error(\phi) \gets running\_error(\phi) + e_{o}(\phi, s, a) - decay(s, \phi)$\;
        \If{$running\_error(\phi) < 0$}{
            $running\_error(\phi) \gets 0$\;
        }
        \If{$running\_error(\phi) > \text{threshold}$}{
            $running\_error(\phi) \gets \text{threshold}$\;
        }
    }

    \If{$running\_error(\phi_{\text{assumed}}) = \text{threshold}$}{
        \textbf{Detect policy switch}\;
        Choose a new assumed opponent policy $\phi_{\text{assumed}}$ with minimum running error\;
        $running\_error(\phi_{\text{assumed}}) \gets running\_error(\phi_{\text{assumed}})/2$;
    }
    Choose an action according to response policy $\rho(\phi_{\text{assumed}})$\;
}
\end{algorithm}

\subsection{Detection of Policy Switch}
We use the observed errors to accumulate a running error for each of the policies from the opponent's policy bank $\Pi_o$ (refer to Algorithm~\ref{alg:opsdemo}, lines 8-24). If the running error of the currently presumed opponent policy $\Phi$ exceeds a threshold value, we assume that the opponent has switched its policy in the meantime. However, naturally occurring errors may make the running error go indefinitely large. Therefore, a decay method of running error is essential.

\subsection{Error Decay}
For a given Markovian state $s$, where the expected error when following policy $\Phi$ is denoted as $e_f$ (refer to Equation~\ref{eq:following_policy}), and the expected error when not following policy $\Phi$ is denoted as $e_{nf}$ (refer to Equation~\ref{eq:not_following_policy}), the decay between these values is defined by Equation \ref{eqn:decay_cal}. Here $\phi^c$ denotes the set of all possible policies for the particular problem.

\begin{equation}
\begin{aligned}
\begin{split}
d & = \alpha e_f + (1 - \alpha) e_{nf} \\
& = \alpha E\left(e_o(\Phi, s)\right) + (1-\alpha)E\left(\frac{e_o(\Phi,s)}{\phi' \in \Phi^{c}}\right)
\end{split}
\end{aligned}
\label{eqn:decay_cal}
\end{equation}

In this equation, the parameter $\alpha \in [0,1]$ represents the strictness factor of the decay. A higher value of $\alpha$ implies a more stringent detection model that disallows policies similar to, but not significantly differing from, the assumed one. Conversely, lower values of $\alpha$ allow for a more lenient approach. The careful choice of $\alpha$ is crucial in tailoring the detection model to specific requirements.

This decay prevents the running error from growing indefinitely and gets calculated dynamically (refer to Algorithm~\ref{alg:opsdemo}, line 10).

\subsection{Identification of the Post-Switch Policy}
In order to reuse the trained response policies effectively, the identification of the opponent's policy after a switch becomes a critical task. We suggest maintaining a record of running errors for all potential opponent policies. When the running error associated with the presently assumed policy surpasses a predetermined threshold, the policy exhibiting the minimum current running error is designated as the switched policy. Subsequently, the running error is halved to mitigate the occurrence of excessively frequent switches (refer to Algorithm~\ref{alg:opsdemo}, line 22). This approach aims to enhance the robustness and stability of policy detection in dynamic environments.

\section{Empirical Evaluation} \label{sec:evaluation}
In this section, we assess the performance of OPS-DeMo within a Markov game, Predator Prey, by comparing it against current state-of-the-art learning algorithms using various metrics. The key metrics for this analysis include the accumulated rewards and the accuracy of assumptions about the opponent’s policy. This assessment aims to specifically evaluate the efficacy of the running error estimation method in the context of frequent changes in the opponent's policy and varying strictness levels.

Distinct from conventional learning algorithms, OPS-DeMo utilizes models that have been trained after the initial learning phase. For the purposes of this evaluation, we exclude active learning components, operating under the assumption of a set of probable opponent policies that change infrequently. The response policies employed are pre-trained using techniques such as PPO. Notably, models such as BPR+ and DPN-BPR+ are excluded from the comparison due to their inapplicability in continuous environments. Additionally, SAM is also omitted due to its undefined decay parameters and ambiguity in policy definition.

\subsection{Implementation}
The experimental setup involves a 2-predator, 2-prey configuration with fully observable environmental states and actions. No direct communication between agents is allowed. This setup accommodates diverse policies for each agent, enabling distinct optimal response policies for varying opponent policies. Rewards are intentionally kept sparse throughout the episode to minimize the information available about the opponent's policy. Rather than assuming the optimality of the opponent's actions, the agent focuses on identifying optimal actions based on its understanding of the opponent's likely behavior. The experimentation is conducted on a machine with an Apple Silicon M2 processor and 8GB of primary memory.

\subsection{Environment Setup}
The setup involves a Predator Prey grid-world environment with two predators and two prey. The objective of the game is for each predator to simultaneously capture one prey, with the aim of capturing both prey in the shortest time possible. This approach seeks to maximize the rewards earned within a single episode. Negative rewards are incurred for each timestep where a predator fails to catch a prey or experiences collisions with other predators. This setup addresses dual objectives: optimizing successful captures and minimizing undesirable events.

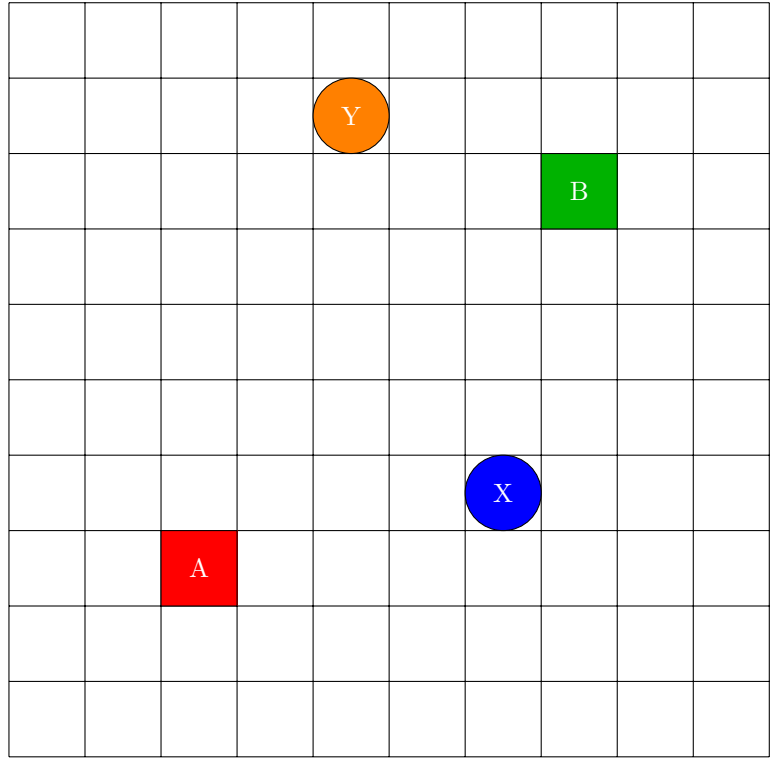
\begin{figure}[ht]
\centering
\begin{tikzpicture}
        \draw[step=1,black,thin] (0,0) grid (10,10);

        \node[draw, fill=red, minimum size=1cm, text=white] at (2.5, 2.5){A};
    \node[draw, fill=green!70!black, minimum size=1cm, text=white] at (7.5, 7.5){B};

        \node[draw, fill=blue, circle, minimum size=1cm, text=white] at (6.5, 3.5){X};
    \node[draw, fill=orange, circle, minimum size=1cm, text=white] at (4.5, 8.5){Y};

                            \end{tikzpicture}
\caption{Predator-Prey Environment}
\label{fig:predator_prey}
\end{figure}

The game setup includes two prey, designated as 'Prey X' and 'Prey Y', which move randomly throughout the environment, relying solely on observations. The first predator is named ``Predator A", and the second is ``Predator B". Similarly, the preys are ``Prey X" and ``Prey Y." Predator B has two probable policies: chasing Prey X or Prey Y, periodically switching. Predator A adapts to these changes, choosing optimal actions based on its belief.

In the training environment for both predators, the reward structure is defined: +100 for catching both prey, -1 for each timestep without adjacent prey, and -1 for colliding with agents. The primary objectives are to maximize Predator A's rewards and accurately update belief regarding Predator B's current policy. In Fig. \ref{fig:predator_prey}, we present a visual representation to illustrate the described Predator Prey scenario.

\subsection{Training Setup}
Our training environment is a $10 \times 10$ predator-prey grid-world created using the OpenAI Gym library~\cite{openai-gym}. In this grid-world, each training episode lasts for a maximum of 40 timesteps. To facilitate the training process, we employ the Stable-Baselines3 library~\cite{stable-baselines3}. Specifically, we train two potential policies for \textbf{Predator B}, with a focus on either chasing Prey A or Prey B. This training utilizes PPO algorithm and ran for up to 1,000,000 iterations. To address the sparse reward issue, we introduce a penalty based on the Manhattan distance between \textbf{Predator B} and its target Prey into the environment-provided reward. Subsequently, we proceed to train a response policy for \textbf{Predator A} for each of the potential policies of \textbf{Predator B} using the PPO algorithm, again reaching 1,000,000 iterations.

\subsection{Simulation of Policy Switch}
Upon deploying the trained models in the environment, we implement a periodic policy switch for \textbf{Predator B} between chasing Prey X and Prey Y. Crucially, this information about \textbf{Predator B}'s current policy is kept concealed from \textbf{Predator A}. \textbf{Predator A} only has access to information about its own rewards and the chosen actions of \textbf{Predator B} at each timestep. \textbf{Predator B} utilizes the online data to calculate the observed error and the corresponding decay of the Markov state, updating the running error. The response policy of \textbf{Predator B} is then chosen from its policy bank based on this information to determine its action for the next timestep.

\subsection{Hyperparameters related to the Experiments}
Among others, we want to focus on the following hyperparameters that we have experimented with in our work,
\begin{enumerate}
    \item \textit{Experiment with Different Strictness Factors}: We conduct experiments with different strictness factors, $\alpha \in \{0.8, 0.9, 0.95, 0.99\}$, to assess their impact on the model's performance. Data is collected on timesteps where the assumed policy aligns with the opponent's concealed policy.
    \item \textit{Experiment with a Standalone PPO-trained Model}: To compare the performance of OPS--DeMo to PPO, we train a \textbf{Predator A} model with PPO. In this setup, \textbf{Predator B} switches its policy every 100 timestamps, with training extending up to 1,000,000 iterations. These trained models are then evaluated, and their performance is compared using accumulated rewards. The standalone PPO-trained model does not have a belief mechanism for predicting an opponent's behavior and only uses the environment state to determine its next action.
\end{enumerate}

Now, we present some empirical results to analyze the novelty and efficiency of our solution.

\subsection{Performance of Running Error Estimation}
We assessed the effectiveness of OPS-DeMo's running error estimation method when \textbf{Predator B} switched its policy every $n$ timesteps.

\begin{figure}[ht]
  \centering
  \includegraphics[width=\textwidth]{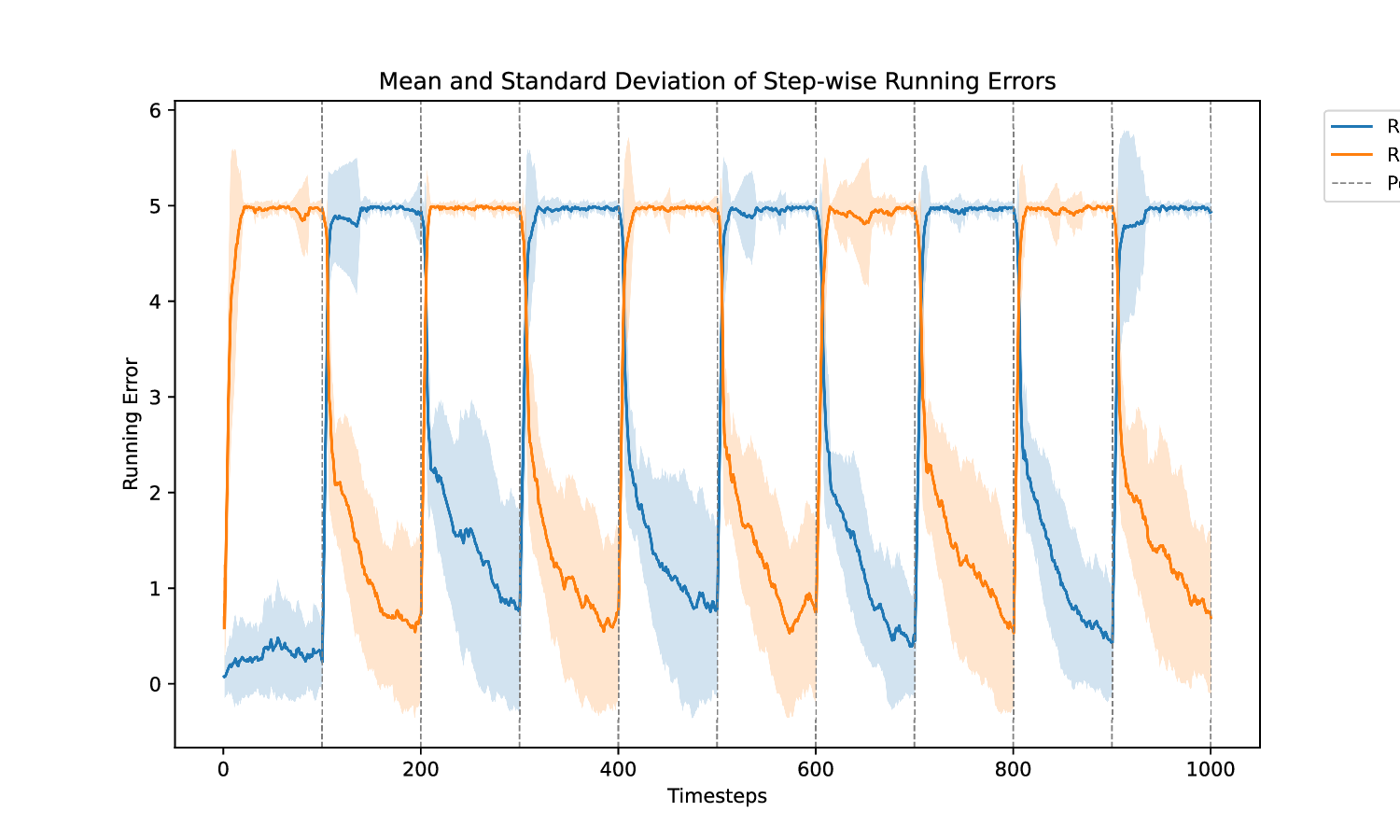}   \caption{Running errors of two probable policies of Predator B, based on observations from Predator A, with Predator B switching its policy every $100$ timesteps.}
  \label{fig:switch100runningerror}
\end{figure}

\begin{figure}[ht]
  \centering
  \includegraphics[width=\textwidth]{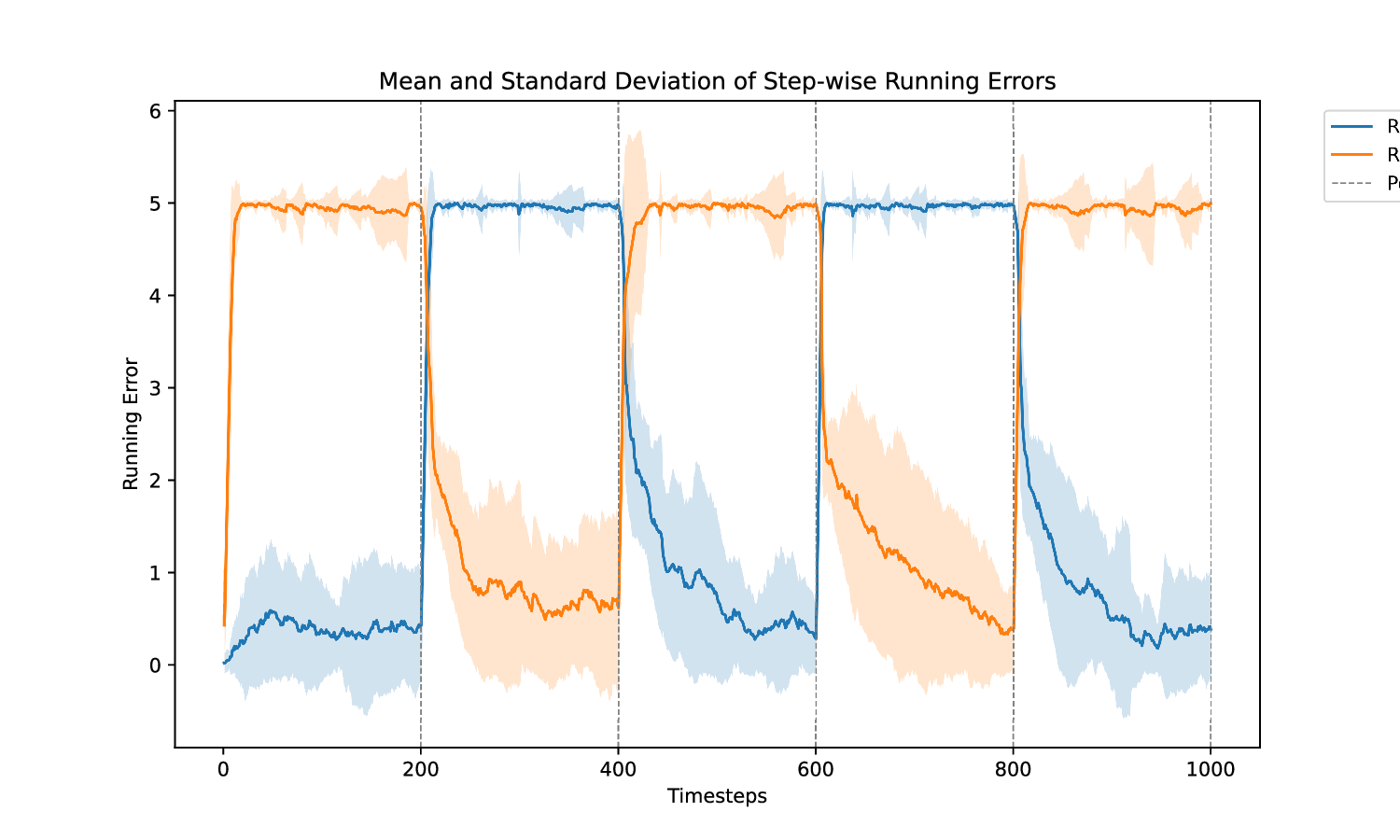}   \caption{Running errors of two probable policies of Predator B, based on observations from Predator A, with Predator B switching its policy every $200$ timesteps.}
  \label{fig:switch200runningerror}
\end{figure}

In both Fig.~\ref{fig:switch100runningerror} and Fig.~\ref{fig:switch200runningerror}, it is evident that the running error remains low when the assumed policy is correct and stays at the threshold when the assumed policy is incorrect. Notably, there is a rapid increase in the running error for the incorrectly assumed policy at the points of policy switch and a relatively slower decline for the correct policy. This is because the error decay is closer to the expected observed error when the opponent is following the Assumed Opponent Policy (AOP), the observed error is much greater than the decay when the opponent is not following AOP. The halving of the running error for the newly assumed policy after detecting the switch contributes to the rapid convergence of the correct assumption to a low running error.

\subsection{Impact of the Strictness Factor}
Our experiment in the Predator-Prey environment involved varying the strictness factors ($\alpha$) while predator B switches its policy every $100$ timestep. We examined how the running errors behaved under different strictness conditions.

\begin{figure}[ht]
  \centering
  \includegraphics[width=0.8\textwidth]{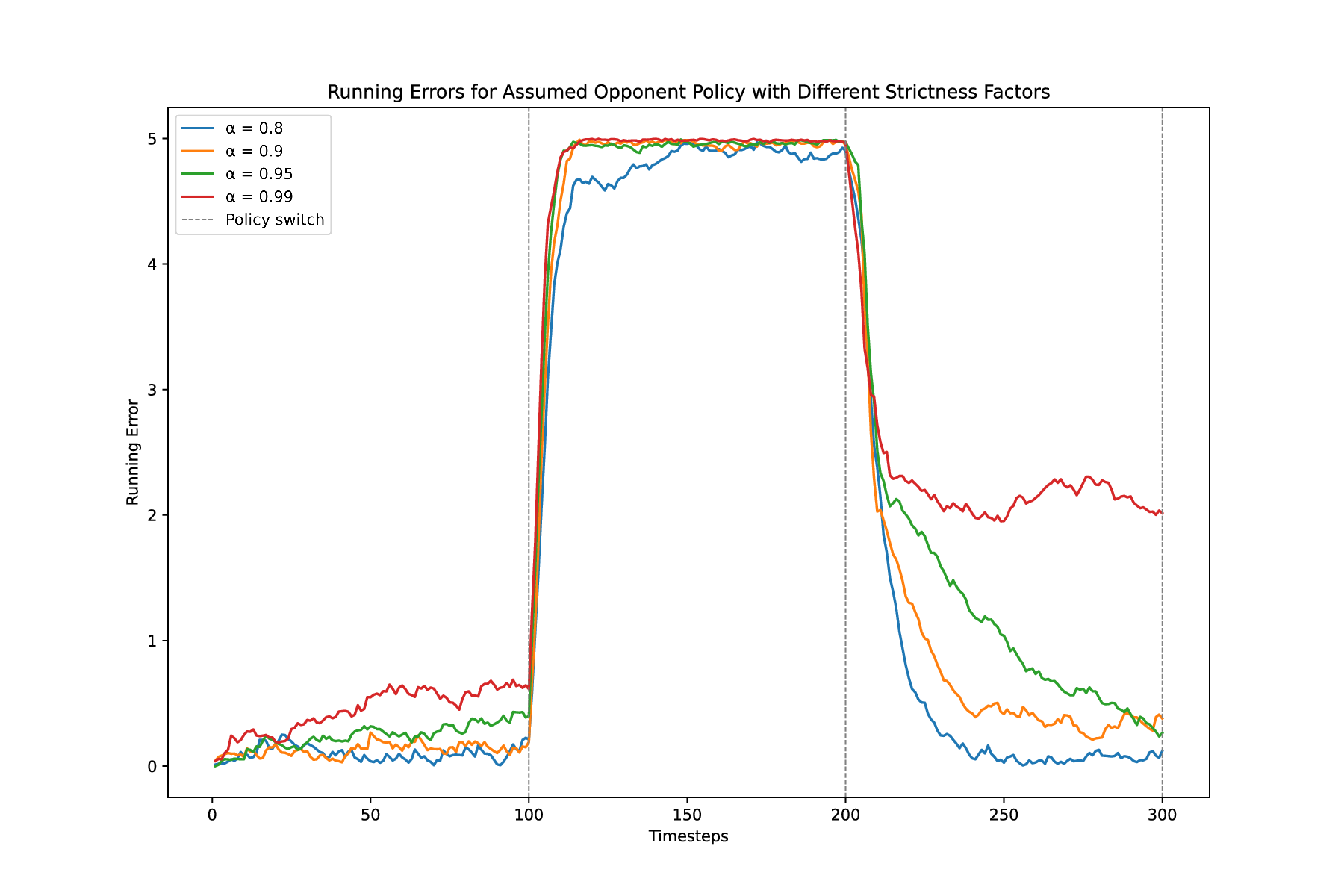}   \caption{Running errors of a probable policy of Predator B, based on observations from Predator A. Predator B switches its policy every $100$ timestep, illustrating the impact of different strictness factors on running errors.}
  \label{fig:strictness_switch100runningerror}
\end{figure}

Fig.~\ref{fig:strictness_switch100runningerror} illustrates that an increased strictness factor results in a quicker rise in running errors following a policy switch by the opponent. However, it also indicates that the reduction in running error after transitioning to that specific AOP is more gradual under higher strictness conditions. This phenomenon occurs because the component added to the running error is typically negative when the opponent adheres to AOP but positive when the opponent deviates from AOP. The closer the decay aligns with the expected observed error while following AOP (Equation~\ref{eq:following_policy}), the smaller the magnitude of the negative value becomes, and the larger the magnitude of the positive value grows.

\subsection{Accuracy of Assumed Opponent Policy}
In our experiment, we varied the strictness factors ($\alpha$) while Predator B switched its policy every $100$ timestep. We assessed the accuracy of the AOP by calculating the ratio of timesteps where the assumed policy matched the actual policy to the total number of timesteps.

\begin{figure}[ht]
  \centering
  \includegraphics[width=0.7\textwidth]{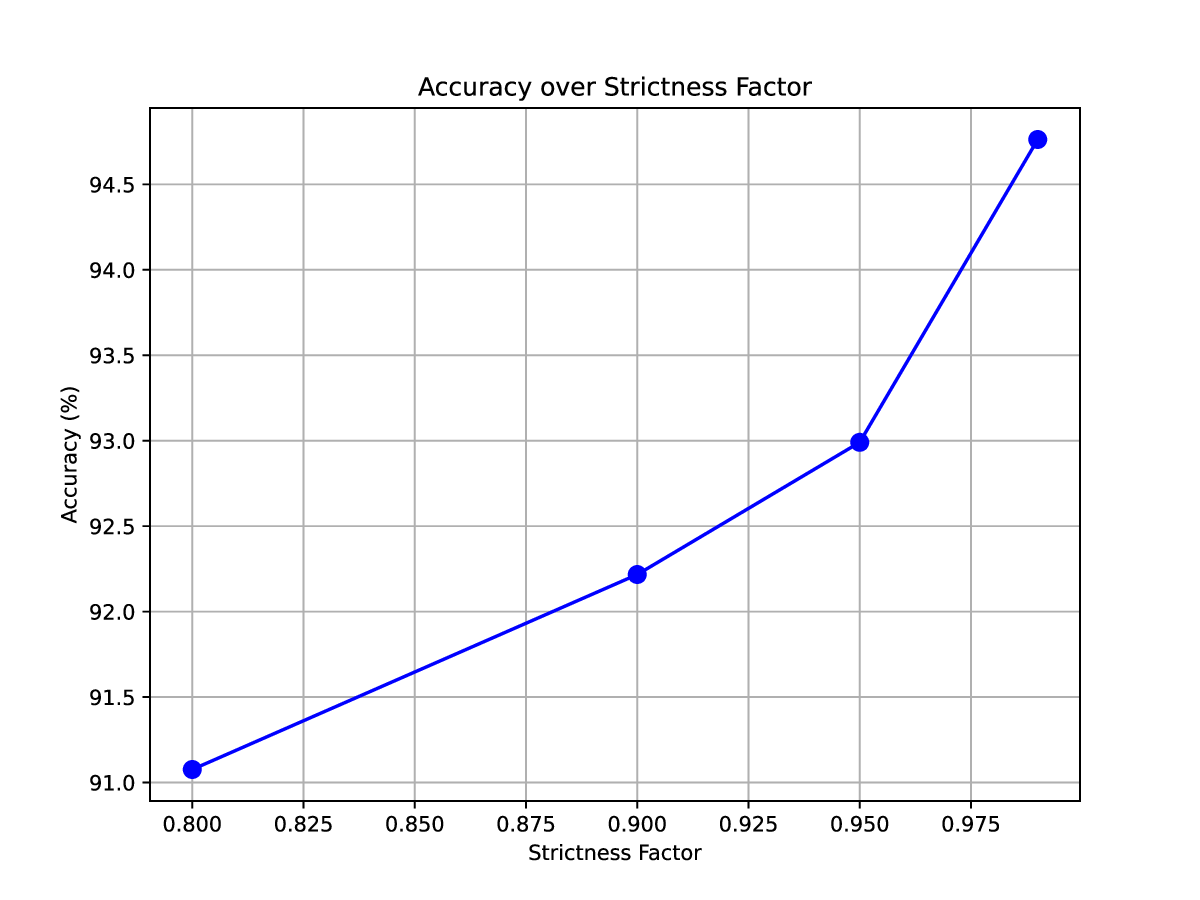}   \caption{Accuracy of Predator B's Assumed Opponent Policy (AOP) based on observations from Predator A. Predator B switches its policy every $100$ timestep, showing the impact of different strictness factors on accuracy.}
  \label{fig:strictness_accuracy_switch100}
\end{figure}

Fig.~\ref{fig:strictness_accuracy_switch100} illustrates that as we increase the strictness factor, the accuracy improves. This suggests that the error estimation method becomes less lenient, becoming more discerning in distinguishing between somewhat similar yet different policies. This is because a higher strictness factor makes the running error rise faster and detects the policy switch earlier. But in the trade-off, natural noises from the environment may get some false positives in that case.

\subsection{Comparison through Episodic Accumulated Rewards}
Based on the experimental data comprising 25 runs, each lasting 1000 episodes, discernible enhancements in accumulated rewards per episode are evident. These improvements stem from enhanced collaborative dynamics between the two predators. Predator A, exhibiting swift adaptation to policy switches by Predator B, formulates responses based on its inferred beliefs about Predator B's current policy.

In Fig.~\ref{fig:reward_histogram}, we observe that while the standalone PPO-trained model performs commendably in most episodes, there are instances where it fails to capture both preys within the defined maximum of 40 timesteps. Consequently, it misses out on the +100 reward due to a lack of collaborative efforts. Conversely, OPS-DeMo, which dynamically detects Predator B's policy during runtime and adjusts its response policy accordingly, demonstrates fewer occurrences of such failures.

\begin{figure}[ht]
  \centering
  \includegraphics[width=0.8\textwidth]{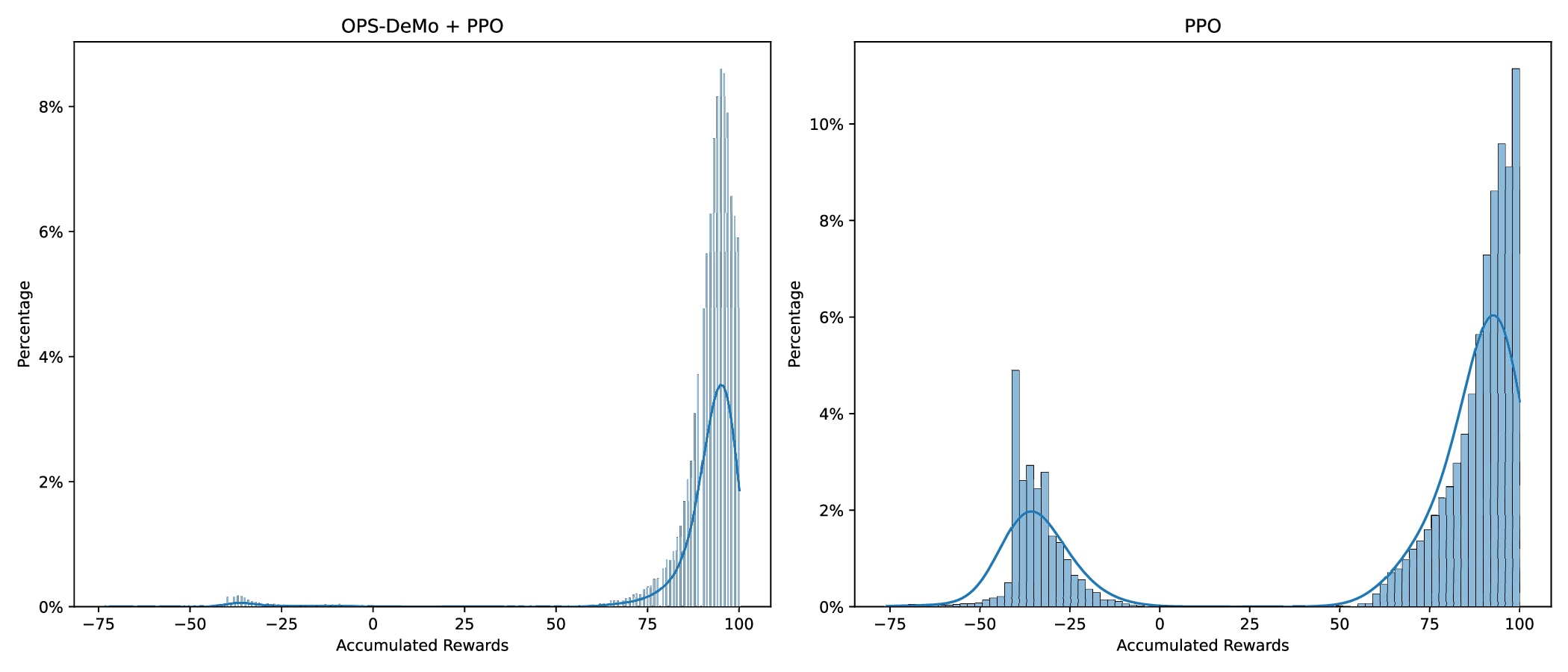}
  \caption{Histogram depicting accumulated rewards per episode by Predator A while Predator B switches its policy every $100$ timesteps.}
  \label{fig:reward_histogram}
\end{figure}

\begin{figure}[ht]
  \centering
  \includegraphics[width=0.5\textwidth]{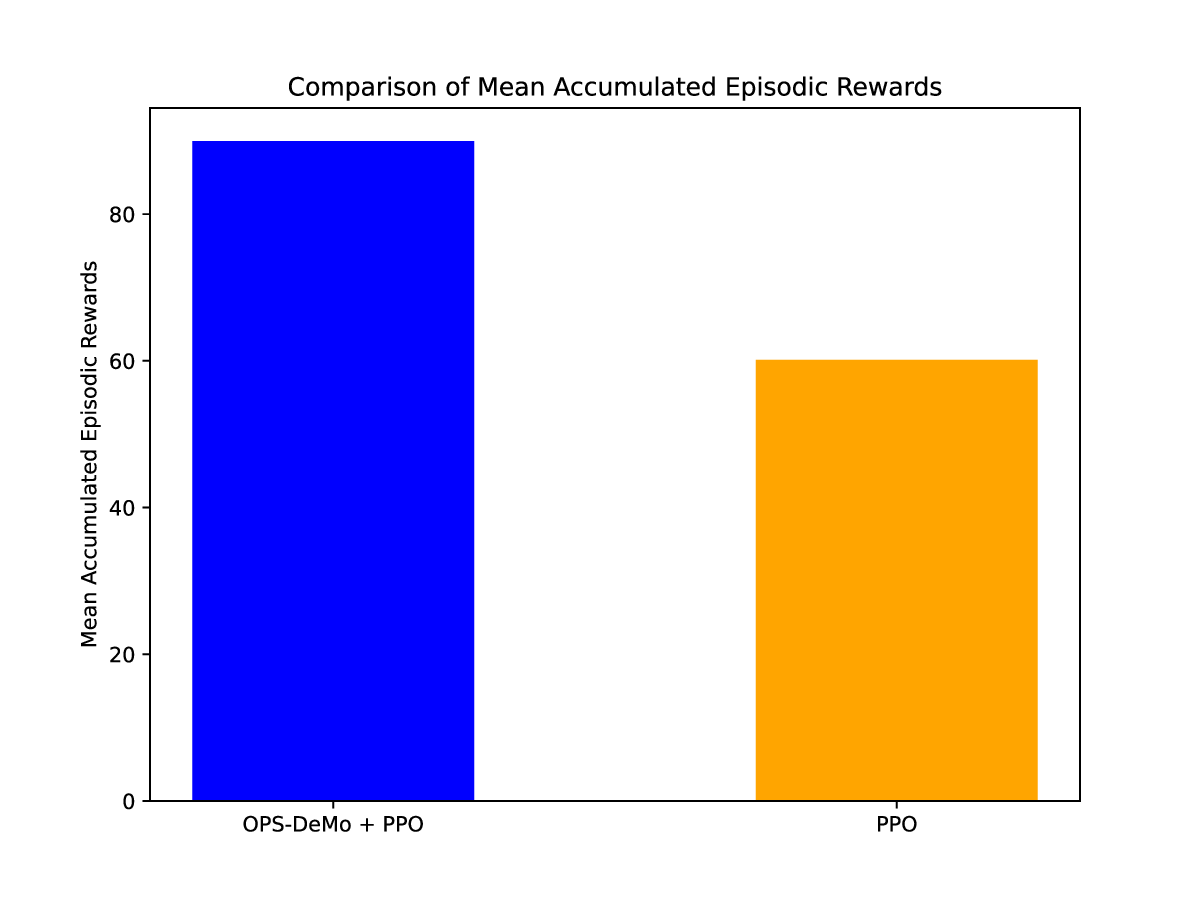}   \caption{Comparison of mean accumulated rewards per episode by Predator A, with Predator B switching its policy every $100$ timesteps, between OPS-DeMo and PPO.}
  \label{fig:reward_bar_chart}
\end{figure}

Fig.~\ref{fig:reward_bar_chart} illustrates that OPS-DeMo achieves a substantial $49.6\%$ improvement over the standalone PPO-trained model in terms of Predator A's mean episodic rewards. This improvement is attributed to OPS-DeMo's robustness in handling increased variance resulting from high uncertainty about Predator B's current policy.

\begin{table}[ht]
  \centering
  \caption{Statistical Summary of Episodic Accumulated Rewards}
  \begin{tabular}{lcc}
    \hline
    \textbf{Algorithm} & \textbf{Mean} & \textbf{Standard Deviation} \\
    \hline
    OPS-DeMo + PPO & 89.9662 & 18.7922 \\
    PPO & 60.1371 & 53.0235 \\
    \hline
  \end{tabular}
  \label{table: reward_comparison}
\end{table}

The consistency observed in OPS-DeMo's rewards, as evidenced by the lower standard deviation in Table~\ref{table: reward_comparison}, results from increased certainty about Predator B's behavior. OPS-DeMo makes informed decisions based on this certainty, contrasting with the standalone PPO-trained model, which tends to overlook recent action data from Predator B.

\section{Conclusions and Future Work} \label{sec:conclusion}
Detecting policy switches in a nonstationary multiagent environment is challenging but advantageous. It is difficult to check compliance when action distributions are uniform or data is limited. Using an error metric comparing observed and expected actions helps address this, with a decay mechanism preventing error escalation. Running error calculations for probable policies helps infer switches, enabling appropriate response policy selection. The proposed OPS-DeMo algorithm uses these methods for detection and response, outperforming standalone PPO models with more consistent rewards per episode and lower standard deviation. In the future, we are planning to work on incorporating continuous learning for more precise opponent policy estimation, developing a robust method for detecting opponent policies with uniform frequency distribution of actions alongside detecting and learning unforeseen opponent policies.

\end{document}